\documentclass{article}

\usepackage[final]{nips_2017}

\usepackage[utf8]{inputenc} 
\usepackage[T1]{fontenc}    
\usepackage{hyperref}       
\usepackage{graphicx}       
\usepackage{url}            
\usepackage{booktabs}       
\usepackage{amsfonts}       
\usepackage{nicefrac}       
\usepackage{microtype}      
\usepackage{amsmath}
\usepackage{amsthm}
\newtheorem{definition}{Definition}
\newtheorem{example}{Example}
\newtheorem{proposition}{Proposition}
\newtheorem{theorem}{Theorem}

\title{How Does Knowledge of the AUC Constrain the Set of Possible Ground-truth Labelings?}

\author{
  Jacob Whitehill\\
  Department of Computer Science\\
  Worcester Polytechnic Institute\\
  Worcester, MA 01609 \\
  \texttt{jrwhitehill@wpi.edu} \\
}

\begin{document}

\maketitle

\begin{abstract}
Recent work on privacy-preserving machine learning has considered how datamining
competitions such as Kaggle could potentially be ``hacked'', either intentionally or inadvertently,
by using information from an oracle that reports a classifier's accuracy on the test set
\citep{blum2015ladder,hardt2014preventing,zheng2015toward,Whitehill2016}.
For binary classification tasks in particular,
one of the most common accuracy metrics is the Area Under the ROC Curve (AUC), and in this paper we explore
the mathematical structure of how the AUC is computed from an $n$-vector of real-valued ``guesses''  with respect to
the ground-truth labels. We show how knowledge 
of a classifier's AUC on the test set can constrain the set of possible
ground-truth labelings, and we derive an algorithm both to 
compute the exact number of such labelings and to enumerate
efficiently over them. Finally, we provide empirical evidence that, surprisingly,
the number of compatible labelings can actually \emph{decrease} as $n$ grows, until a test set-dependent threshold is reached.
\end{abstract}

\section{Introduction and Related Work}
Datamining contests such as Kaggle and KDDCup can accelerate progress in many application domains by
providing standardized datasets and a fair basis of comparing multiple algorithmic approaches.
However, their utility will diminish if the integrity of leaderboard rankings is called
into question due to either intentional or accidental overfitting to the test data.
Recent research on privacy-preserving machine learning \citep{Whitehill2016,blum2015ladder,zheng2015toward}
has shown how information on the accuracy of a contestant's guesses, returned to the 
contestant by an oracle, can divulge information
about the test data's true labels. Such oracles are often provided by the organizers of the competition themselves.
For example, in the 2017 Intel \& MobileODT Cervical Cancer Screening  competition\footnote{
\url{https://www.kaggle.com/c/intel-mobileodt-cervical-cancer-screening}}
hosted by Kaggle, every contestant can submit her/his guesses up to 5 times per day, and for each submission the oracle returns the
log-loss of the guesses with respect to the ground-truth values of the entire $512$-element test set. The contestant
can use the accuracy information to improve (hopefully) the classifier design and then re-submit.

{\bf AUC}: For binary classification problems, one of the most commonly used accuracy metrics is the
{\bf A}rea Under the {\bf R}eceiver Operating Characteristics {\bf C}urve (AUC). In contrast
to other accuracy metrics such as log-loss and 0/1 loss, which can be computed as the sum of example-wise
losses over each example in the test set, the AUC statistic is computed over all possible \emph{pairs} of test examples,
such that 
each pair contains one example from each class. In a recent paper, \cite{Whitehill2016} 
showed that an oracle that provides contestants with information on the AUC of their guesses can inadvertently
divulge information on the ground-truth labels of the test examples.
As a concrete example, suppose that a tiny test set contains just 4 examples; a
contestant's real-valued guesses for these labels is $\hat{\bf y} = (0.2, 0.5, 0.9, 0.1)$;
and an oracle informs the contestant that her/his guesses have achieved an AUC of exactly
$0.75=3/4$. How does this information constrain the set of possible binary ground-truth vectors for the test set?
In this example, it turns out that there is \emph{exactly one} possible ground-truth vector -- namely
${\bf y} = (1,0,1,0)$ -- for which
the AUC of the contestant's guesses is exactly $0.75$. Hence, based on a single oracle query, the contestant has
managed to deduce the test labels with complete certainty.
This simple example raises more general questions: For a test set with $n$ examples and a fixed AUC
of $c=p/q$ (where $p,q\in \mathbb{Z}$), how many compatible binary ground-truth vectors are there? Does this number grow monotonically
in $n$, or might there exist some ``pathological'' combinations of the number of test examples $n$,
number of positively labeled  examples $n_1$, and the contestant's AUC $c$, such that this number is small?
If the number is small, can the solution candidates be enumerated efficiently?
This paper explores these questions in some detail.

{\bf Related work}:
Over the past few years there has been growing theoretical and practical interest in the statistical validity
of scientific results that are obtained from \emph{adaptive} data analyses, in which the results of one experiment
inform the design of the next \citep{dwork2015preserving,hardt2014preventing}. For the particular application of
datamining contests -- in which contestants can submit their guesses to an oracle, receive information on their accuracy, revise
their guesses, and resubmit --
a potential danger is that the rankings and associated accuracy statistics of different contestants
may be unreliable. Therefore, the design of algorithms to generate contest leaderboards that are robust to ``hacking'',
whether intentional as part of an attack or
inadvertently due to adaptive overfitting, has begun to generate significant research interest \citep{blum2015ladder,zheng2015toward}.
In particular,
\cite{blum2015ladder} proposed an algorithm (``Ladder'')
that can reliably estimate the accuracy of a contestant's classifier on the true test data distribution,
even when the classifier has been adaptively optimized based on the output of an oracle on the empirical test distribution.

While the availability of an oracle in datamining contests presents potential problems, it is also
useful for helping contestants to focus their
efforts on more promising algorithmic approaches. Our research is thus related to privacy-preserving machine learning
and differential privacy (e.g., \cite{Dwork2011,ChaudhuriMonteleoni2009,BlumEtAl2013}), which are concerned with how to provide
useful aggregate statistics without disclosing private information about particular examples in the dataset.
The AUC statistic, in particular, has been investigated in the context of privacy:
\cite{StoddardEtAl2014} proposed an algorithm for computing ``private ROC''
curves and associated AUC statistics. \cite{Matthews2013} showed how an attacker who already knows most of the test labels can
estimate the remaining labels if he/she gains access to an empirical ROC curve, i.e., 
a set of classifier thresholds and corresponding true positive and false positive rates.

The prior work most similar to ours is by \cite{Whitehill2016}. They
showed a weak form of lower bound on the number of possible
binary ground-truth vectors ${\bf y} \in \{0,1\}^n$ for which the contestant's guesses $\hat{\bf y}$ achieve
any fixed AUC $c$. Specifically, for every AUC value $c=p/q\in (0,1)$, there exists an infinite sequence  of dataset
sizes ($n = 4q, 8q, 12q, \ldots$) such that the number of satisfying ground-truth vectors ${\bf y}\in \{0,1\}^n$ grows exponentially in
$n$. However, this result does not preclude the possibility that there might be certain pathological cases -- 
combinations of $p$, $q$, $n_0$, and $n_1$ -- for which the number of satisfying ground-truth vectors is actually much smaller.
Conceivably, there might be values of $n$ that lie between integer multiples of $4q$ for which the number of satisfying solutions is small.
Moreover, the lower bound in \cite{Whitehill2016} applies only to datasets that contain at least $4q$ examples and says nothing
about smaller (but possibly still substantial) datasets.

{\bf Contributions}: The novel contributions of our paper are the following:
(1) We derive an algorithm to compute the exact number of
$n$-dimensional binary ground-truth vectors
for which a contestant's real-valued vector of guesses achieves a fixed AUC, along
with an algorithm to efficiently generate all such vectors.  (2) We show that 
the number of distinct binary ground-truth vectors, in which
$n_1$ entries are $1$, and for which  a contestant's guesses achieve a fixed AUC,
is equal to the number of elements in a \emph{truncated} $n_1$-dimensional discrete simplex 
(i.e., a subset of $\Delta^{n_1}_{d}$). (3) We provide empirical evidence that the number of satisfying binary ground-truth vectors
can actually \emph{decrease} with increasing $n$, until a test set-dependent threshold is reached.

\section{Notation and Assumptions}
Let ${\bf y} = (y_1,\ldots,y_n) \in \{ 0,1 \}^n$ be the ground-truth binary labels
of $n$ test examples, and let $\hat{\bf y} = (\hat{y}_1,\ldots,\hat{y}_n) \in \mathbb{R}^n$ be the
contestant's real-valued guesses. Let $\mathcal{L}_1({\bf y})=\{i: y_i=1\}$
and $\mathcal{L}_0({\bf y})=\{i: y_i=0\}$ represent the index sets of the examples that are labeled $1$ and $0$, respectively.
Similarly define $n_1({\bf y})=|\mathcal{L}_1({\bf y})|$ and $n_0({\bf y})=|\mathcal{L}_0({\bf y})|$ to be the number of examples labeled $1$ and $0$ in ${\bf y}$, respectively. For brevity, we sometimes write simply $n_1$, $n_0$, $\mathcal{L}_0$, or $\mathcal{L}_1$ if the argument to these functions is
clear from the context.

We assume that the contestant's guesses 
$\hat{y}_1,\ldots,\hat{y}_n$ are all \emph{distinct} (i.e., $\hat{y}_i=\hat{y}_j \iff i=j$).
In machine learning applications where classifiers analyze high-dimensional, real-valued feature vectors, this is common.

Importantly, but without loss of generality, we assume that the test examples are ordered according to
$\hat{y}_1,\ldots,\hat{y}_n$, i.e., $\hat{y}_i > \hat{y}_j \iff i > j$. This significantly simplifies the notation.

Finally, we assume that the oracle provides the contestant with \emph{perfect} knowledge of the AUC $c=p/q$, where $p/q$ is a reduced
fraction (i.e.,  the greatest common factor of $p$ and $q$ 1) on the \emph{entire} test set,
and that the contestant knows both $p$ and $q$. 

\section{AUC Accuracy Metric}
The AUC has two mathematically equivalent definitions \citep{Tyler2000,AgarwalEtAl2005}: (1)
the AUC is the Area under the Receiver Operating Characteristics (ROC) curve, which plots
the true positive rate against the false positive rate of a classifier  on some test set.
The ROC thus characterizes the performance of the classifier over all possible thresholds on its real-valued output, and the
AUC is the integral of the ROC over all possible false positive rates in the interval $[0,1]$. (2)
The AUC represents the fraction of \emph{pairs} of test examples -- one labeled $1$ and one labeled $0$ -- in which the classifier
can correctly identify the positively labeled example based on the classifier output.  Specifically, since we assume
that all of the contestant's guesses are distinct, then the AUC can be computed as:
\begin{equation}
\label{eqn:auc}
\textrm{AUC}({\bf y}, \hat{\bf y}) = \frac{1}{n_0 n_1} \sum_{i\in\mathcal{L}_0} \sum_{j\in\mathcal{L}_1} \mathbb{I}[\hat{y}_i < \hat{y}_j]
\end{equation}
Equivalently, we can define the AUC in terms of the number of \emph{misclassified pairs} $h$:
\[
\textrm{AUC}({\bf y}, \hat{\bf y}) = 1 - \frac{h({\bf y}, \hat{\bf y})}{n_0 n_1}
\]
where
\[
h({\bf y},\hat{\bf y}) = \sum_{i\in\mathcal{L}_0} \sum_{j\in\mathcal{L}_1} \mathbb{I}[\hat{y}_i > \hat{y}_j]
\]
As is evident in Eq.~\ref{eqn:auc}, all that matters
to the AUC is the \emph{relative ordering} of the $\hat{y}_i$, not their exact values. 
Also, if all examples belong to the same class and either $n_1=0$ or $n_0=0$, then the AUC is undefined.
Finally, the AUC is a \emph{rational} number because it can be written as the fraction of two integers $p$ and $q$,
where $q$ must divide $n_0 n_1$.

Since we assume (without loss of generality) that the contestant's guesses are ordered such that $\hat{y}_i < \hat{y}_j \iff i <j $, then
we can simplify the definition of $h$  to be:
\begin{equation}
\label{eqn:h_simple}
h({\bf y},\hat{\bf y}) = \sum_{i\in\mathcal{L}_0} \sum_{j\in\mathcal{L}_1} \mathbb{I}[i>j]
\end{equation}

\section{Computing the Exact Number of Binary Labelings for which AUC=$c$}
In this paper, we are interested in determining the number of unique binary vectors ${\bf y} \in \{0,1\}^n$ such that
the contestant's guesses $\hat{\bf y} \in \mathbb{R}^n$ achieve a fixed AUC of $c$. The bulk of the effort is to
derive a recursive formula for the number of unique binary vectors with a \emph{fixed} number $n_1$  of $1$s that give the
desired AUC value.

{\bf Intuition}: Given a real-valued vector $\hat{\bf y}$ representing
the contestant's guesses and a corresponding binary vector ${\bf y}$ representing the ground-truth test labels, the number $h({\bf y}, \hat{\bf y})$
of misclassified
pairs of examples (such that each pair contains one example from each class) can be increased by $1$ by
``left-swapping'' any occurrence of $1$ in ${\bf y}$ (at index $j'$) with a $0$ that occurs immediately to the left  of it
(i.e., at index $j'-1$) -- see Figure \ref{fig:leftswap}.
To generate a vector ${\bf y}$ such that $h({\bf y}, \hat{\bf y})=d$ for any desired $d\in\{0,\ldots,q\}$,
we start with a vector ${\bf r}$
in ``right-most configuration'' -- i.e., where all the $0$s occur to the left of all the $1$s -- because (as we will show)
$h({\bf r}, \hat{\bf y})=0$. We then apply a sequence of multiple left-swaps to each of the $1$s in ${\bf r}$, and count
the number of ways of doing so such that the total number is $d$.
Because we want to determine the number of \emph{unique} vectors ${\bf y}$ such that 
$h({\bf y}, \hat{\bf y})=d$, we restrict the numbers $s_1,\ldots,s_{n_1}$ of left-swaps applied to the $n_1$ different $1$s in ${\bf r}$ so that
$s_i\geq s_j$ for all $i<j$. This results in a proof that the number of possible ground-truth binary labelings, for any given
value of $n_1$ and for which a given vector of guesses misclassifies $d$ pairs of examples, is equal to the number of points in a
$n_1$-dimensional discrete simplex $\Delta^{n_1}_{d}$ that has been truncated by the additional
constraint that $n_0 \geq s_1\geq \ldots \geq s_{n_1}$.

To get started, we first define ``left-swap'' and ``right-most configuration'' more precisely:
\begin{definition}
\label{def:sigma}
For any ${\bf y} \in \{0,1\}^n$ and $i\in \{2,\ldots,n\}$ where $y_i=1$ and $y_{i-1}=0$,
define the (partial) function $\sigma: \{0,1\}^n \times \mathbb{Z}^+ \rightarrow \{0,1\}^n$ such that
$\sigma({\bf y},i)=(y_1,\ldots,y_i,y_{i-1},y_{i+1},\ldots,y_n)$.
Function $\sigma$ is said to perform a \textbf{left-swap on ${\bf y}$ from index $i$}.
\end{definition}

\begin{definition}
\label{def:rho}
For any ${\bf y} \in \{0,1\}^n$, $i\in \{2,\ldots,n\}$, and $k < i$ where $y_i=1$ and $y_{i-1}=\ldots=y_{i-k}=0$,
define the (partial) function $\rho: \{0,1\}^n \times \mathbb{Z}^+ \times \mathbb{Z}^+ \rightarrow \{0,1\}^n$,
where
\[
\rho({\bf y},i,k)=\left\{ \begin{array}{rl}\sigma({\bf y}, i) & \textrm{for $k=1$} \\
                                   \underbrace{\sigma(\ldots ( \sigma(\sigma}_{\textrm{$k$}}({\bf y}, i), i-1), \ldots), i-(k-1)) & \textrm{for $k>1$}
                  \end{array} \right.
\]
Function $\rho$ is said to perform \textbf{$k$ consecutive left-swaps on ${\bf y}$ from index $i$}.
\end{definition}
\begin{example}
Let ${\bf y}=(y_1, y_2, y_3, y_4, y_5) \in \{0,1\}^5$. Then $\sigma({\bf y}, 4) = (y_1, y_2, y_4, y_3, y_5)$ and 
$\rho({\bf y}, 4, 3) = (y_4, y_1, y_2, y_3, y_5)$.
\end{example}

\begin{definition}
Let ${\bf y}$ be a vector of $n$ binary labels such that $n_1$ entries are $1$. Let
$\mathcal{L}_1({\bf y})=\{p_1,\ldots,p_{n_1}\}$, ordered such that $p_i<p_j \iff i<j$,
be the indices of the $1$s in the vector. Then we say ${\bf y}$
is in a \textbf{right-most configuration} iff $p_i = n - n_1 + i$ for every $i \in \{ 1, \ldots, n_1 \}$.
\end{definition}

\begin{proposition}
\label{prop:h_zero}
Let ${\bf r}=(r_1,\ldots,r_n)$ be a binary vector of length $n$ in right-most configuration such that $n_1$ entries are $1$. 
Let $\hat{\bf y}=(\hat{y}_1,\ldots,\hat{y}_n)$ be a vector of $n$ real-valued guesses, ordered such that
$\hat{y}_i<\hat{y}_j \iff i<j$. Then $h({\bf r}, \hat{\bf y})=0$.
\end{proposition}
\begin{proof}
Since ${\bf r}$ is in right-most configuration, it is clear that the right-hand side of
\[
h({\bf r}, \hat{\bf y}) = \sum_{i\in\mathcal{L}_0({\bf r})} \sum_{j\in\mathcal{L}_1({\bf r})} \mathbb{I}[i>j]
\]
sums to $0$.
\end{proof}

\begin{figure}
\begin{center}
\includegraphics[width=5in]{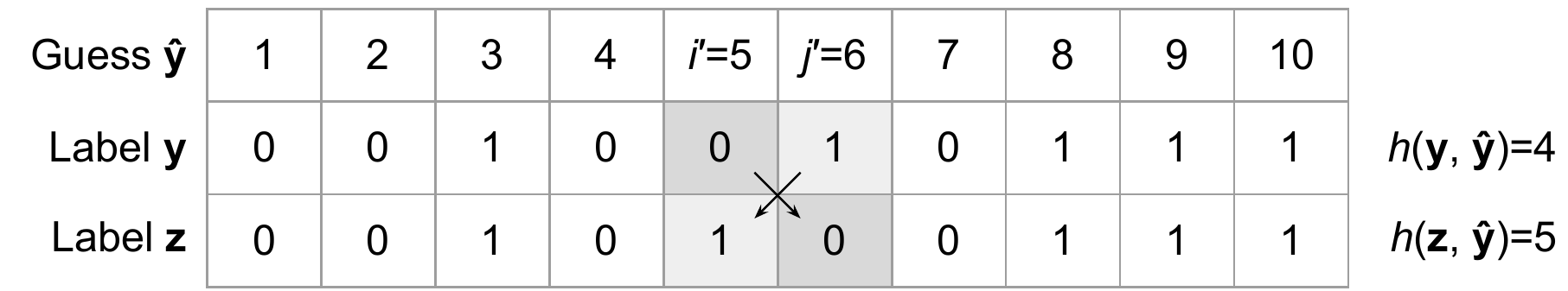}
\end{center}
\caption{Illustration of how performing a left-swap on binary vector ${\bf y}$ at index $j'$ yields a new vector ${\bf z}$
such that the number of misclassified pairs $h({\bf z}, \hat{\bf y})$ is one more than  $h({\bf y}, \hat{\bf y})$.
Specifically, $\hat{\bf y}$ misclassifies pairs $(3,4)$, $(3,5)$, $(3,7)$, and $(6,7)$ w.r.t.~to ${\bf y}$,
since for each such pair $(i,j)$, $\hat{y}_i<\hat{y}_j$ but $y_i>y_j$. In contrast,
$\hat{\bf y}$ misclassifies $(3,4)$, $(3,6)$, $(3,7)$, $(5,6)$, and $(5,7)$ w.r.t.~to ${\bf z}$.}
\label{fig:leftswap}
\end{figure}

\begin{proposition}
\label{prop:h_one_more}
Let ${\bf y}$ be a vector of $n$ binary labels, and let ${\bf z} = \sigma({\bf y}, j')$
be another vector of binary labels that is produced by a single left-swap of ${\bf y}$ at index $j'\in \{2,\ldots,n\}$, where
$y_{j'}=1$ and $y_{i'}=0$, and $i'=j'-1$.
Let $\hat{\bf y}$ be a vector of real-valued guesses.
Then the number of pairs misclassified by $\hat{\bf y}$ w.r.t.~${\bf z}$ is one more than the
number of pairs misclassified by $\hat{\bf y}$ w.r.t.~${\bf y}$ -- i.e.,
$h({\bf z}, \hat{\bf y}) = h({\bf y}, \hat{\bf y}) + 1$.
\end{proposition}
\begin{proof}
To shorten the notation, let $\mathcal{L}_0 = \mathcal{L}_0({\bf y}), \mathcal{L}_1 = \mathcal{L}_1({\bf y})$, and let
$\tilde{\mathcal{L}}_0 = \mathcal{L}_0({{\bf z}}), \tilde{\mathcal{L}}_1 = \mathcal{L}_1({{\bf z}})$.
We can split the summation in Equation \ref{eqn:h_simple} into four sets of pairs (see Figure \ref{fig:leftswap}): 
(a) those involving neither $i'$ nor $j'$; (b) those involving $j'$ but not $i'$; (c) those involving $i'$ but not $j'$;
and (d) the single pair involving both $i'$ and $j'$. By grouping the pairs this way, we obtain:
\begin{eqnarray*}
\label{eqn:h_ytilde_split}
\lefteqn{h({\bf z},\hat{\bf y})} && \\
  && = \left(\sum_{\substack{i\in\\ \tilde{\mathcal{L}}_0\setminus\{j'\}}} \sum_{\substack{j\in\\ \tilde{\mathcal{L}}_1\setminus\{i'\}}} \mathbb{I}[i>j]\right) +
       \left(\sum_{\substack{i\in\\ \tilde{\mathcal{L}}_0\setminus\{j'\}}} \mathbb{I}[i>i']\right) +
       \left(\sum_{\substack{j\in\\ \tilde{\mathcal{L}}_1\setminus\{i'\}}} \mathbb{I}[j'>j]\right) +
       \mathbb{I}[j'>i']
\end{eqnarray*}
Notice that $\tilde{\mathcal{L}}_0 = (\mathcal{L}_0({\bf y}) \setminus \{ i' \}) \cup \{ j' \}$, and
hence $\mathcal{L}_0({\bf y})\setminus\{i'\} = \tilde{\mathcal{L}}_0({\bf z}) \setminus \{ j' \}$. Similarly,
$\tilde{\mathcal{L}}_1\setminus\{i'\} = \mathcal{L}_1\setminus\{j'\}$.
Then we have:
\[
h({\bf z},\hat{\bf y})
     = \sum_{\substack{i\in\\ \mathcal{L}_0\setminus\{i'\}}} \sum_{\substack{j\in\\ \mathcal{L}_1\setminus\{j'\}}} \mathbb{I}[i>j] +
       \sum_{\substack{i\in\\ \mathcal{L}_0\setminus\{i'\}}} \mathbb{I}[i>i'] +
       \sum_{\substack{j\in\\ \mathcal{L}_1\setminus\{j'\}}} \mathbb{I}[j'>j] \quad +
       \mathbb{I}[j'>i'] \\
\]
Since $i'+1=j'$, then there cannot exist any index $i \in \mathcal{L}_0\setminus\{i'\}$ whose value is ``between''
$i'$ and $j'$; in other words,
$i>i' \iff i>j'$ for every $i\in \mathcal{L}_0\setminus\{i'\}$. Similarly,
$j'>j \iff i'>j$ for every $j\in \mathcal{L}_1\setminus\{j'\}$. Hence:
\[
h({\bf z},\hat{\bf y})
     = \sum_{\substack{i\in\\ \mathcal{L}_0\setminus\{i'\}}} \sum_{\substack{j\in\\ \mathcal{L}_1\setminus\{j'\}}} \mathbb{I}[i>j] +
       \sum_{\substack{i\in\\ \mathcal{L}_0\setminus\{i'\}}} \mathbb{I}[i>j'] +
       \sum_{\substack{j\in\\ \mathcal{L}_1\setminus\{j'\}}} \mathbb{I}[i'>j] \quad +
       \mathbb{I}[j'>i'] \\
\]
Finally, since $\mathbb{I}[j'>i'] = 1$ and $\mathbb{I}[i'<j']=0$, then:
\begin{eqnarray*}
h({\bf z},\hat{\bf y})
     &=& \sum_{\substack{i\in\\ \mathcal{L}_0\setminus\{i'\}}} \sum_{\substack{j\in\\ \mathcal{L}_1\setminus\{j'\}}} \mathbb{I}[i>j] +
       \sum_{\substack{i\in\\ \mathcal{L}_0\setminus\{i'\}}} \mathbb{I}[i>j'] +
       \sum_{\substack{j\in\\ \mathcal{L}_1\setminus\{j'\}}} \mathbb{I}[i'>j] \quad +
       \mathbb{I}[i'>j'] + 1 \\
     &=& h({\bf y}, \hat{\bf y}) + 1
\end{eqnarray*}
\end{proof}

\begin{proposition}
\label{prop:Y_S_correspondence}
Suppose a dataset contains $n$ examples, of which $n_1$ are labeled $1$ and $n_0=n-n_1$ are labeled $0$.
Let $\mathcal{Y}_{n_1} = \{ {\bf y}\in\{0,1\}^n: \sum_i y_i = n_1 \}$, and
let $\mathcal{S}_{n_1} = \{ (s_1,\ldots,s_{n_1})\in \mathbb{Z}^{n_1}: n_0\geq s_1 \geq \ldots \geq s_{n_1} \geq 0 \}$.
Then $\mathcal{Y}_{n_1}$ and $\mathcal{S}_{n_1}$ are in 1-to-1 correspondence.
\end{proposition}
\begin{proof}
Every binary vector ${\bf y}\in \mathcal{Y}_{n_1}$ of length $n$, of which $n_1$ entries are $1$, can be described by a unique 
vector of integers in the set $\mathcal{P}_{n_1}=\{(p_1,\ldots,p_{n_1})\in \mathbb{Z}^+: 1\leq p_1<\ldots<p_{n_1} \leq n\}$
specifying the indices of the $1$s in ${\bf y}$ in increasing order. In particular,
$\mathcal{Y}_{n_1}$ and $\mathcal{P}_{n_1}$ are in 1-to-1 correspondence with a bijection
$f_p: \mathcal{Y}_{n_1} \rightarrow \mathcal{P}_{n_1}$. Hence, if we can show a bijection $f_s: \mathcal{P}_{n_1} \rightarrow \mathcal{S}_{n_1}$,
then we can compose $f_s$ with $f_p$ to yield a new function $f: \mathcal{Y}_{n_1} \rightarrow \mathcal{S}_{n_1}$; since the
composition of two bijections is bijective, then $f$ will be bijective.

We can construct such an $f_s$ as follows:
\[
f_s(p_1,\ldots,p_{n_1}) = (s_1,\ldots,s_{n_1})\quad\textrm{where}\ s_i = n - n_1 + i - p_i
\]
We must first show that $(s_1,\ldots,s_{n_1})=f_s(p_1,\ldots,p_{n_1}) \in \mathcal{S}_{n_1}$ for every $(p_1,\ldots,p_{n_1})\in\mathcal{P}_{n_1}$; in 
particular, we must show that $n_0\geq s_1\geq \ldots \geq  s_{n_1} \geq 0$.
Since every $p_i$ is an integer, we have that $p_{j} - p_i \geq j - i$
for every $j > i$. Hence:
\[
n - n_1 + p_{j} - p_i \geq n - n_1 + j - i
\]
We then add $i$ to and subtract $p_j$ from both sides to obtain:
\begin{eqnarray*}
n - n_1 + i - p_i &\geq& n - n_1 + j - p_{j}\\
s_i &\geq& s_{j}\quad\forall j>i
\end{eqnarray*}
The two boundary cases are $s_{n_1}$:
\[
s_{n_1}=n-n_1+n_1-p_{n_1}=n-p_{n_1}\geq 0
\]
and $s_1$:
\[
s_1=n-n_1+1-p_1=n_0 + 1-p_1 \leq n_0
\]
\begin{figure}
\begin{center}
\includegraphics[width=5in]{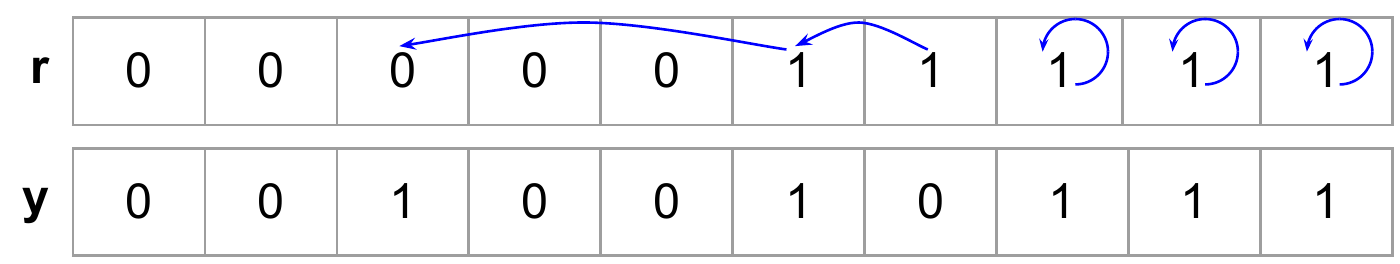}
\end{center}
\caption{Illustration of how any binary vector ${\bf y}$ with $n_1$ $1$s can be produced by repeatedly
left-swapping the $1$'s in a right-most binary vector ${\bf r}$. In the example above, left-swaps are indicated with
blue arrows, with $s_1=3,s_2=1$, and $s_3=s_4=s_5=0$.}
\label{fig:production_from_r}
\end{figure}

{\bf $f_s$ is 1-to-1}: Suppose $(s_1,\ldots,s_{n_1})=f_s(p_1,\ldots,p_{n_1})$ and
$(s_1',\ldots,s_{n_1}')=f_s(p_1',\ldots,p_{n_1}')$, and suppose $s_i=s_i'$ for each $i$. Then:
\begin{eqnarray}
n - n_1 + i - p_i &=& n - n_1 + i - p_i'\\
&\iff&\\
p_i &=& p_i'
\end{eqnarray}
for each $i$.

{\bf $f_s$ is onto}: For every $(s_1,\ldots,s_{n_1}) \in \mathcal{S}_{n_1}$, we can find
$(p_1,\ldots,p_{n_1})$ such that $f_s(p_1,\ldots,p_{n_1}) = (s_1,\ldots,s_{n_1})$ by setting
$p_i = n - n_1 + i - s_i$ for each $i$. It only remains to be shown that
$1\leq p_1 < \ldots < p_{n_1} \leq n$: For any $j>i$,
\[ p_j - p_i = j - i + s_i - s_j \]
Since $s_i \geq s_j$ (by definition of $\mathcal{S}_{n_1}$), we have:
\begin{eqnarray*}
p_j - p_i &\geq& j - i \\
          &>& 0
\end{eqnarray*}
and hence $p_j > p_i$. Moreover,
\begin{eqnarray*}
p_1 &=& n - n_1 + 1 - s_1 \\
    &\geq & n - n_1 + 1 - n_0 \\
    &\geq& 1
\end{eqnarray*}
and
\begin{eqnarray*}
p_{n_1} &=& n - n_1 + n_1 - s_{n_1} \\
	&\leq& n
\end{eqnarray*}
\end{proof}

\begin{theorem}
\label{prop:Y_partition_number}
Suppose a dataset contains $n$ examples, of which $n_1$ are labeled $1$ and $n_0$ are labeled $0$.
Let $\mathcal{Y}_{n_1}^{(d)} = \{ {\bf y}\in\{0,1\}^n: \sum_i y_i = n_1\ \wedge\ h({\bf y}, \hat{\bf y}) = d \}$, and
let $\mathcal{S}_{n_1}^{(d)} = \{ (s_1,\ldots,s_{n_1})\in \mathbb{Z}^{n_1}: n_0\geq s_1 \geq \ldots \geq s_{n_1} \geq 0\ \wedge\ \sum_i s_i = d \}$.
Then $\mathcal{Y}_{n_1}^{(d)}$ and $\mathcal{S}_{n_1}^{(d)}$ are in 1-to-1 correspondence.
\end{theorem}
\begin{proof}
Since $\mathcal{Y}_{n_1}^{(d)} \subset \mathcal{Y}_{n_1}$ and $\mathcal{S}_{n_1}^{(d)} \subset \mathcal{S}_{n_1}$,
and since $\mathcal{Y}_{n_1}$ and $\mathcal{S}_{n_1}$ are in 1-to-1 correspondence with bijection $f$
(from Proposition \ref{prop:Y_S_correspondence})
then we must show only that the image of $\mathcal{Y}_{n_1}^{(d)}$ through $f$ is $\mathcal{S}_{n_1}^{(d)}$.
Suppose that $(s_1,\ldots,s_{n_1}) = f({\bf y})$ for some ${\bf y}\in \mathcal{Y}_{n_1}^{(d)}$.
Then we can write ${\bf y}$ as
\[
{\bf y} = \underbrace{\rho(\rho(\ldots(\rho}_{n_1}({\bf r}, n-n_1+1, s_1), n-n_1+2, s_2) \ldots), n, s_{n_1})
\]
In other words, ${\bf y} \in \mathcal{Y}_{n_1}^{(d)}$ can be obtained from ${\bf r}$ (a binary vector of length
$n$, such that $n_1$ elements are labeled $1$, in right-most configuration) by performing a sequence of consecutive
left-swaps on the $1$s in ${\bf r}$.
To see this, observe that the first $1$ in ${\bf r}$ is always immediately
preceded by $n_0$ $0$s; hence, we can perform $s_1\leq n_0$ consecutive left-swaps on ${\bf r}$ from
index $n-n_1+1$. (See Figure \ref{fig:production_from_r} for an illustration.)
Moreover, after performing these consecutive left-swaps,
then the second $1$ in ${\bf r}$  will be immediately preceded by $s_1$ $0$s;
hence, we can perform $s_2\leq s_1$ consecutive left-swaps on ${\bf r}$ from
index $n-n_1+2$. After performing these consecutive
left-swaps, then the third $1$ in ${\bf r}$ will be immediately preceded by $s_2$ $0$s;
and so on. After performing the consecutive left-swaps for each of the $1$s in ${\bf r}$,
then the position of the $i$th $1$ in the resulting
vector is $n - n_1 + i - s_i = p_i$ for each $i$, as desired. 

Next, recall that,
by Proposition \ref{prop:h_zero}, $h({\bf r}, \hat{\bf y})=0$.  Moreover,
by Proposition \ref{prop:h_one_more}, each left-swap increases the value of $h$ by $1$; hence, applying
$\rho$ to perform $s_i$ consecutive left-swaps increases the value of $h$ by $s_i$. Summing over all
$1$s results in a total of $\sum_{i=1}^{n_1} s_i$ misclassified pairs, i.e., $h({\bf y}, \hat{\bf y}) = \sum_{i=1}^{n_1} s_i$.
But since ${\bf y} \in \mathcal{Y}_{n_1}^{(d)}$, we already know that $h({\bf y}, \hat{\bf y}) = d$. Therefore,
$\sum_i s_i = d$, and hence $(s_1,\ldots,s_{n_1})\in\mathcal{Y}_{n_1}^{(d)}$.

\end{proof}
Interestingly, the set $\mathcal{S}_{n_1}^{(d)}$ is a discrete $n_1$-dimensional simplex
$\Delta^{n_1}_{d}$ that has been truncated by the additional constraint that $n_0 \geq s_1 \geq \ldots \geq s_{n_1}$.

\subsection{Summing over all possible $n_1$}
Based on Theorem \ref{prop:Y_partition_number}, we can compute the number, $v(n_0, n_1, d)$, of binary vectors of length $n=n_0+n_1$,
such that $n_1$ of the entries are labeled $1$ and for which $h({\bf y}, \hat{\bf y})=d$.
Recall that the AUC can be computed by dividing the number $d$ of misclassified pairs by the total number of example-pairs
$n_0 n_1$. Hence, to compute the \emph{total} number, $w(n,c)$, of binary vectors of length $n$ 
for which $\textrm{AUC}({\bf y}, \hat{\bf y})=c$, we must first determine the set $\mathcal{N}_1$ of possible values for $n_1$,
and then sum $v(n_0,n_1,d)$ over every value in $\mathcal{N}_1$ and the corresponding value $d$.

Suppose that the oracle reports an AUC of $c=p/q$, where $p/q$ is a reduced fraction, Since $c$ 
represents the fraction of all pairs of examples -- one from each class -- that are classified  by the contestant's
guesses correctly, then $q$ must divide the total number ($n_0 n_1$) of pairs
in the test set. Hence:
\[
\mathcal{N}_1 = \{ n_1: (0<n_1<n)\ \wedge\ (q\ |\ (n - n_1) n_1) \}
\]
Since it is possible that $q < n_0 n_1$, we must scale $(q-p)$ by $n_0 n_1 / q$ to determine the actual number of 
misclassified pairs $d$. In particular, we define
\[
d(n_1) = (q-p) n_0 n_1 / q = (q - p) (n - n_1) n_1 / q
\]

Based on $\mathcal{N}_1$ and $m$, we can finally compute:
\[
w(n,c) = \left| \bigcup_{n_1 \in \mathcal{N}_1} \mathcal{S}_{n_1}^{(d(n_1))} \right|
       = \sum_{n_1 \in \mathcal{N}_1} v(n-n_1, n_1, d(n_1))\quad \textrm{since the $\mathcal{S}_{n_1}^{(d(n_1))}$ are disjoint.}
\]

\section{Recursion Relation}
\label{sec:recursion}
We can derive a recursion relation for $v(n_0, n_1, d)$ as follows:
Given any binary vector ${\bf r}$ of length $n$, with $n_1$ $1$s, in right-most configuration,
we can apply $k \in \{0, 1, \ldots, \min(d,n_0)\}$ left-swaps on ${\bf r}$ from index $n-n_1+1$ (i.e., from
the left-most $1$) to yield ${\bf y} = \rho({\bf r}, n-n_1+1, k)$.
Then the vector $(y_{n-n_1-k+2}, y_{n-n_1-k+3}, y_{n-n_1-k+4}, \ldots, y_n)$ (i.e., the last $n_1-1+k$ elements of ${\bf y}$)
consists of $k$ $0$s followed by $(n_1-1)$ $1$s; in other words, it is in right-most configuration.
Thus, by iterating over all possible $k$ and computing for each choice how many \emph{more} left-swaps are necessary to reach
a total of $d$, we can define $v$ recursively:
\[
v(n_0, n_1, d) = \sum_{k=0}^{\min(d,n_0)} v(k, n_1-1, d-k)
\]
with initial conditions:
\begin{eqnarray*}
v(n_0, n_1, 0) &=& 1\quad\forall n_0\geq 0, n_1\geq 0\\
v(0, n_1, d) &=& 0\quad\forall n_1\geq 0, d>0\\
v(n_0, 0, d) &=& 0\quad\forall n_0\geq 0, d>0\\
\end{eqnarray*}
Dynamic programming using a three-dimensional memoization table can be used to compute $v$
in time $O(n_0 n_1 d)$. Moreover, the recursive algorithm above can also be used \emph{constructively} (though
with large space costs)
to compute the set of all binary vectors ${\bf y}$ of length $n$, of which $n_1$ are $1$,
such that $h({\bf y}, \hat{\bf y})=d$ for any $d$; conceivably, this could be useful for
performing some kind of attack that uses the set of all compatible binary ground-truth vectors to improve
the contestant's accuracy \citep{Whitehill2016}. In order to apply this construction, the test examples
must first be sorted in increasing value of the contestant's guesses; the constructive algorithm is then applied to
generate all possible ${\bf y}$; and then the components of each of the possible binary vectors must be reordered
to recover the original order of the test examples.

\section{Growth of $w(n,c)$ in $n$ for fixed $c$}
\cite{Whitehill2016} showed that, for every fixed rational $c=p/q \in (0,1)$, the number of possible
binary ground-truth vectors for which the contestant's guesses achieve AUC of exactly $c$, grows
exponentially in $n$. However, their result applies only to datasets that are at least
$n=4q$ in size.  What can happen for smaller $n$?

Using the recursive formula from Section \ref{sec:recursion}, we found empirical evidence that
$w(n,c)$ may actually be (initially) monotonically \emph{decreasing} in $n$, until $n$ reaches a threshold (specific
to $q$) at which it begins to increase again. As an example with $p=1387$ and $q=1440$ (and hence
$d=1440-1387=53$), we can
compute the number of possible binary labelings that are compatible
with an AUC of exactly $c=p/q=1387/1440$ (which is  approximately $96.3\%$) as a function of
$n$:

\begin{center}
\begin{tabular}{c|c|c}
$n$ & $\mathcal{N}_1$ & $w(n,c)$\\\hline
$76$ & $\{ 36, 40 \}$ & $657488$ \\
$77$ & $\{ 32, 45 \}$ & $654344$ \\
$78$ & $\{ 30, 48 \}$ & $650822$ \\
$84$ & $\{ 24, 60 \}$ & $622952$ \\
$92$ & $\{ 20, 72 \}$ & $572728$ \\
$98$ & $\{ 18, 80 \}$ & $529382$ \\
$106$ & $\{ 16, 90 \}$ & $468686$ \\
\end{tabular}
\end{center}

Here, $w(n,c)$ decreases steadily until $n=106$.
We conjecture  that $w(n,c)$ is monotonically non-increasing
in $n$ for $n\leq \min\{ n_0+n_1: n_0 n_1 = 2q \}$, for every fixed $c$.

While the number of satisfying solutions in this example for $n=106$ is still in the hundreds of thousands,  it is 
easily small enough to allow each possibility to be considered individually, e.g., as part of
some algorithmic attack to maximize performance within a datamining competition \citep{Whitehill2016}.
Furthermore, we note that test sets on the order of hundreds of examples are not uncommon -- 
the 2017 Intel \& MobileODT Cervical Cancer Screening is one example.

\section{Summary \& Future Work}
We have investigated the mathematical structure of how the Area Under the Receiver Operating Characteristics Curve
(AUC) accuracy metric is computed from the binary vector of ground-truth labels and a real-valued vector
of guesses. In particular, we derived an efficient recursive algorithm with which to count the exact number
of binary vectors for which the AUC of a fixed vector of guesses is some value $c$; we also derived a constructive
algorithm with which to enumerate all such binary vectors.

In future work it would be interesting to examine whether
and how knowledge of the possible ground-truth labelings could be exploited to improve an existing vector of guesses;
a simple mechanism was proposed by \cite{Whitehill2016}, but it is practical only  for tiny datasets. In addition,
it would be useful to explore how \emph{multiple} subsequent oracle queries  might be used to prune
the set of possible ground-truth labelings more rapidly.

\bibliographystyle{aaai}
\bibliography{paper}

\end{document}